\documentclass{article}

\usepackage{iclr2027_conference,times}
\iclrfinalcopy

\usepackage{microtype}
\usepackage{graphicx}
\usepackage{subcaption}
\usepackage{booktabs}
\usepackage{amsmath,amssymb}
\usepackage{wrapfig}
\usepackage{caption}
\usepackage{enumitem}
\usepackage{hyperref}
\usepackage{mathtools}
\usepackage{amsthm}
\usepackage{multirow}
\usepackage{float}
\usepackage{tabularx}
\usepackage{makecell}
\usepackage{svg}
\newtheorem{theorem}{Theorem} 
\newtheorem{lemma}{Lemma}
\newtheorem{assumption}{Assumption}

\newtheorem{definition}[theorem]{Definition}

\DeclareUnicodeCharacter{22EF}{\(\cdots\)}

\usepackage{titlesec}
\titlespacing{\section}{0pt}{4pt plus 4pt minus 2pt}{4pt plus 4pt minus 2pt}
\titlespacing{\subsection}{0pt}{2pt plus 0pt minus 2pt}{2pt plus 0pt minus 2pt}
\title{Parameter-Efficient \\ Temporal Domain Generalization}

\author{\textbf{Zekun Cai}$^{1}$ \quad
    \textbf{Xinyuan Song}$^{2}$ \quad
    \textbf{Yiheng Yao}$^{1}$ \quad
    \textbf{Xuan Song}$^{3}$ \\
    \textbf{Ryosuke Shibasaki}$^{4}$ \quad
    \textbf{Liang Zhao}$^{2}$ \\
    $^{1}$The University of Tokyo, Tokyo, Japan \quad
    $^{2}$Emory University, Atlanta, GA, USA \\
    $^{3}$Jilin University, Changchun, China \quad
    $^{4}$LocationMind, Tokyo, Japan \\
    \texttt{caizekun@csis.u-tokyo.ac.jp, \{xinyuan.song,liang.zhao\}@emory.edu} \\
    \texttt{songxuan@jlu.edu.cn, shiba@locationmind.com} \\
}
\begin{document}

\maketitle


\vspace{-2em}
\begin{abstract}
Temporal domain generalization (TDG) learns from a sequence of observed domains to predict models for unseen future domains. Existing parameter-dynamics methods become prohibitively expensive at modern backbone scales because their temporal encoders and generators operate on the entire adapted parameter state. Replacing full-model updates with LoRA does not remove this failure mode: the dynamics model must still encode and reconstruct every adapter parameter, making the auxiliary model itself prohibitively large for modern backbones. We propose \underline{Lo}w-rank \underline{Ti}me-varying \underline{C}ore (LoTiC), a parameter-efficient TDG method that decomposes temporal updates into shared low-rank bases and a compact core. The bases define a stable update subspace, while only the core evolves. LoTiC therefore learns temporal dynamics in an $r\times r$ space instead of tracking the full adapter. We justify this factorization through the geometry of parameter increments and support continuous, recurrent, and time-conditioned core models. Experiments span six controlled and real-world benchmarks, three data modalities, classification and generation, and backbones from 66M to 8B parameters. LoTiC ranks first in 14 of 15 settings, including near-perfect future-domain accuracy with an 8B backbone on Rotating 2-Moons. It adds only 0.13--0.51\% trainable parameters beyond LoRA and trains 1.8--7.1$\times$ faster than parameter-dynamics baselines. These results establish that temporal parameter dynamics can remain both learnable and efficient as pretrained backbones scale. 
\end{abstract}

\section{Introduction}

Many deployment environments change in an ordered rather than arbitrary manner. News topics follow current events, scientific terminology develops over years, and visual appearance shifts across decades. Domain generalization (DG)~\citep{tremblay2018training,huang2020self,du2020learning,qiao2020learning,wang2022generalizing} prepares a model for unseen domains, but usually treats the observed domains as an unordered collection. Temporal domain generalization (TDG)~\citep{li2020sequential,nasery2021training,qin2022generalizing,bai2023temporal,yong2023continuous,zeng2024generalizing,cai2024continuous,yu2025learning} instead uses their temporal order to learn how the predictive function evolves and to continue that evolution into future domains.

Parameter-dynamics methods provide a direct way to achieve this goal. DRAIN~\citep{bai2023temporal}, Koodos~\citep{cai2024continuous}, and FreKoo~\citep{yu2025learning} learn recurrent, continuous-time, or spectral dynamics over a sequence of model parameters. Their temporal state, however, remains tied to the parameters being evolved. This becomes a decisive limitation for pretrained backbones. Replacing full-model updates with LoRA adapters reduces the state substantially, but does not remove the dependence. The auxiliary model must still encode, evolve, and decode the complete adapter. As the backbone grows, the dynamic state grows with it, and a simple temporal law must be recovered from an increasingly large parameter space.

This scaling problem should not imply that temporal structure disappears in large models. Model performance still deteriorates as deployment moves away from the training period~\citep{luu2022time,lazaridou2021mind}, while models adapted at nearby times exhibit aligned update directions and coherent weight-space trajectories~\citep{nylund2024time,dziadzio2025merge}. Parameter-efficient fine-tuning offers a useful starting point because a pretrained model at time $t$ can be written as $W_{\mathrm{pre}}+\Delta W_t$. Yet independently learned LoRA factors do not provide a stable coordinate system for the update trajectory. Their non-unique factorizations can describe similar updates with different bases, leaving the temporal model to separate genuine evolution from changes of parameterization.

These observations expose three challenges for parameter-efficient TDG. \textbf{1) A geometric foundation.} It is unclear whether the temporal geometry of full-model optima is preserved when the model is represented through parameter-efficient updates. \textbf{2) A scalable temporal state.} LoRA reduces each update, but the complete adapter can still contain millions of parameters, and its non-unique factors obscure the evolution that should be modeled. \textbf{3) A general dynamics interface.} Temporal shifts may be smooth, path-dependent, or directly indexed by time. Supporting these regimes should not require rebuilding the adapter or enlarging the evolved state.

We propose \underline{Lo}w-rank \underline{Ti}me-varying \underline{C}ore (LoTiC) to make the evolving object explicit and compact. LoTiC learns shared low-rank bases that define a common update subspace across domains and concentrates temporal variation in a small core. Future models are obtained by continuing only the core sequence. This construction reduces the modeled state from the complete LoRA parameter vector to a compact representation and admits different dynamics without changing the adapter. We study continuous flow, recurrent, and direct time-to-core parameterizations. Our analysis further shows why temporal geometry is preserved in the parameter-increment space and when shared bases provide a stable approximation to domain-wise updates.

We evaluate LoTiC against eight baselines on six benchmarks covering controlled and natural shifts, tabular, visual, and textual inputs, and classification and generation. The backbones range from 66M to 8B parameters. LoTiC ranks first in 14 of 15 settings. On the controlled benchmarks, it recovers the prescribed evolution of the decision function and maintains over 97\% future-domain accuracy on Rotating 2-Moons even with an 8B backbone after observing only 15 domains. On the real-world benchmarks, it leads in nine of ten settings and improves ROUGE-L F1 by up to 16.3\% on temporal news generation. These gains require only 0.13--0.51\% more trainable parameters than ordinary LoRA, and LoTiC trains 1.8--7.1$\times$ faster than parameter-dynamics baselines.

Our contributions are threefold.
\begin{itemize}[leftmargin=1.3em,itemsep=1pt,topsep=2pt,parsep=0pt]
    \item We formulate TDG in the parameter-increment space and show that translating a low-dimensional manifold of model optima preserves its dimension and embedded structure.
    \item We introduce a shared-basis factorization and derive sufficient stability conditions that confine temporal evolution to an $r\times r$ core with LoRA-like overhead.
    \item We instantiate the core with continuous, recurrent, and time-conditioned dynamics and establish its effectiveness across six benchmarks, eight baselines, and backbones up to 8B parameters.
\end{itemize}

\section{Related Work}
\paragraph{Temporal domain generalization.}
Domain generalization (DG) learns from multiple source domains without access to target-domain data~\citep{tremblay2018training,huang2020self,du2020learning,qiao2020learning,wang2022generalizing}. TDG additionally uses the chronological order of source domains to generalize to future ones. GI~\citep{nasery2021training} regularizes temporal changes through a Taylor expansion. DRAIN~\citep{bai2023temporal} recurrently generates future parameters, TKNets~\citep{zeng2024generalizing} linearizes feature dynamics with Koopman theory, Koodos~\citep{cai2024continuous} models continuous and irregularly sampled parameter trajectories, and FreKoo~\citep{yu2025learning} separates their trend and periodic components. These methods differ in their temporal mechanisms, but their encoders or generators remain coupled to the complete parameter state being evolved. This coupling makes parameter dynamics prohibitive for modern pretrained backbones, even when full-model updates are replaced by adapters.

\paragraph{Temporal adaptation of pretrained models.}
Fine-tuned language models encode temporal information in weight space~\citep{nylund2024time}. TEA~\citep{liu2025scaling} combines temporal experts in a PCA-induced subspace, while TIME~\citep{dziadzio2025merge} merges models along initialization and deployment axes. PromptTDG~\citep{hosseini2025scalable} instead predicts future prompts from their temporal sequence. These approaches show that temporal adaptation can exploit pretrained models, but they do not factor the evolving update into a shared spatial subspace and a compact temporal state. LoTiC addresses this distinction directly.

\paragraph{Parameter-efficient fine-tuning.}
PEFT adapts pretrained models with a small trainable state. Its main forms are additive modules~\citep{houlsby2019parameter,mahabadi2021parameter}, selective updates~\citep{liu2022few,sung2021training}, low-dimensional reparameterizations~\citep{hu2022lora,wu2024mixture}, and hybrid designs~\citep{mao2022unipelt,zhou2024autopeft}. We build on LoRA~\citep{hu2022lora}, which represents each weight update with two low-rank factors. LoTiC is not a new static adapter. It introduces a shared coordinate system in which only a compact core evolves over time.

\section{Problem Definition}
We consider \(T\) observed domains indexed by ordered timestamps \(t_1<\cdots<t_T\). Domain \(\mathcal{D}_{t}=\{(x_i^{(t)},y_i^{(t)})\}_{i=1}^{N_t}\) contains samples from a time-dependent distribution \(P_t(X,Y)\). Changes in \(P_t\) may affect the input distribution, label prevalence, or input--output relationship. Unlike standard DG, the ordering is informative: training may use only \(\{\mathcal{D}_{t_1},\ldots,\mathcal{D}_{t_T}\}\), whereas data from any future domain \(\mathcal{D}_{\tau}\), \(\tau>t_T\), are unavailable.

Let \(W_{\mathrm{pre}}\) be a pretrained backbone. At time \(t\), an update \(\Delta W(t)\) defines the predictor \(g_{W_{\mathrm{pre}}+\Delta W(t)}\). The objective of TDG is to learn the evolution of these updates from the observed sequence and produce \(\Delta W(\tau)\) for unseen future timestamps \(\tau>t_T\). The same formulation covers classification and generation, discrete and continuous timestamps, and different pretrained architectures.


\section{Methodology}
In this section, we develop LoTiC for parameter-efficient temporal domain generalization. We first motivate a low-dimensional manifold structure for the trajectory of optimal updates under distribution shift. We then introduce a manifold-constrained factorization that separates time-invariant spatial bases from a time-varying temporal core. This makes the temporal state independent of the full adapter dimension. Finally, we describe continuous, recurrent, and time-conditioned parameterizations of the core.

\subsection{The Geometry of Parameter Increments for Temporal Domains}
\label{sec:param_increment}
Following prior parameter-dynamics TDG~\citep{cai2024continuous,cai2025continuous,yu2025learning}, we model temporal shifts as a structured latent process and assume that the corresponding full-parameter optima \(\{W_t^*\}_{t=1}^T\) are contained in a low-dimensional embedded manifold \(\mathcal{M}\). This assumption captures coherent physical, social, or semantic evolution rather than arbitrary changes between domains. We now state its consequence in the parameter-increment space.

\begin{definition}[Parameter increments]
\label{def:param_increments}
Fix a pretrained initialization \(W_{\mathrm{pre}}\in\mathbb{R}^p\).
For each time step \(t\in[T]\), define the optimal increment
 \begin{equation}
\Delta W_t^*  :=  W_t^* - W_{\mathrm{pre}} .
 \end{equation}
Define the map \(\phi:\mathbb{R}^p\to\mathbb{R}^p\) by
 \begin{equation}
\phi(W)  :=  W - W_{\mathrm{pre}} .
 \end{equation}
\end{definition}

\begin{lemma}[Embedded manifolds]
\label{lem:translation_manifold}
Let \(\mathcal{M}\subset \mathbb{R}^p\) be an embedded \(m\)-dimensional submanifold. Then \(\mathcal{M}' := \phi(\mathcal{M}) = \mathcal{M}-W_{\mathrm{pre}}\) is also an embedded \(m\)-dimensional submanifold of \(\mathbb{R}^p\).
\end{lemma}

\begin{proof}
The map \(\phi(W)=W-W_{\mathrm{pre}}\) is smooth and bijective with smooth inverse \(\phi^{-1}(U)=U+W_{\mathrm{pre}}\); hence \(\phi\) is a diffeomorphism on \(\mathbb{R}^p\). The image of an embedded submanifold under a diffeomorphism is an embedded submanifold of the same dimension. Therefore \(\mathcal{M}'=\phi(\mathcal{M})\) is an embedded \(m\)-dimensional submanifold.
\end{proof}

\begin{lemma}[Parameter-increment manifold]
\label{lem:peft_manifold}
Given \(\{W_t^*\}_{t=1}^T \subset \mathcal{M}\) for some embedded \(m\)-dimensional submanifold \(\mathcal{M}\subset\mathbb{R}^p\). Then \(\{\Delta W_t^*\}_{t=1}^T \subset \mathcal{M}'\), where \(\mathcal{M}' := \mathcal{M}-W_{\mathrm{pre}}\) is also an embedded \(m\)-dimensional submanifold.
\end{lemma}

\begin{proof}
By Definition~\ref{def:param_increments}, \(\Delta W_t^*=\phi(W_t^*)\).
Since \(W_t^*\in \mathcal{M}\) for each \(t\), it follows that
\(\Delta W_t^*\in \phi(\mathcal{M})=\mathcal{M}'\). The manifold claim follows from
Lemma~\ref{lem:translation_manifold}.
\end{proof}

This result motivates working in the parameter-increment space: modeling temporal generalization via \(\Delta W_t\) provides a principled reparameterization for large pretrained models under temporal domain shifts, and it supports extrapolation by learning trajectories on a low-dimensional manifold.

\begin{figure}[t]
    \centering
    \includegraphics[width=0.98\textwidth]{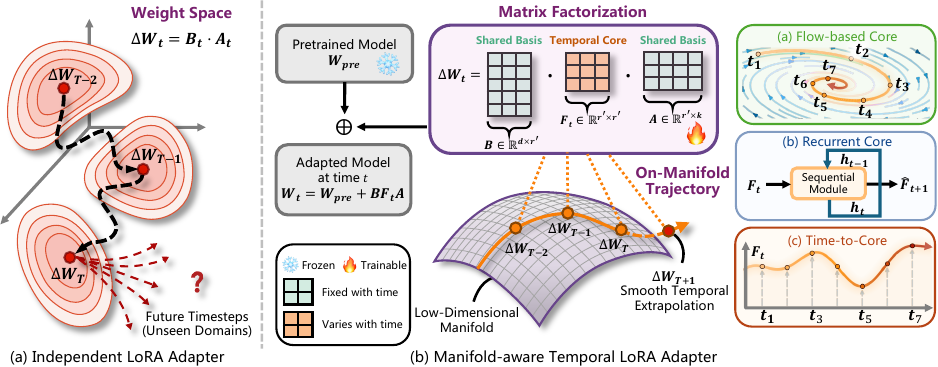}
    \caption{(a) Independent domain-wise LoRA optimization can yield scattered updates without temporal consistency. (b) With shared bases \(B\) and \(A\), temporal adaptation is captured by the compact core \(F_t\). The temporal core can be instantiated by different sequence modeling modules.}
    \label{fig:mat}
\end{figure}

\subsection{Manifold-Constrained Low-Rank Factorization}

Prior methods project complete parameter states into compressed representations and learn dynamics between them~\citep{bai2023temporal,cai2024continuous}. On modern pretrained backbones, however, the required encoder and decoder remain coupled to the full adapted state and become prohibitively large. Directly generating full weights may also compromise pretrained capabilities~\citep{luo2025empirical,li2024revisiting}. We instead model temporal dynamics inside the low-rank update space induced by PEFT.


\subsubsection{LoRA as a Basis for Temporal Generalization}

In a standard PEFT pipeline, each temporal domain is adapted independently with its own LoRA factors \((B_t,A_t)\). While this provides domain-specific flexibility, it imposes no mechanism to coordinate the learned updates across time. For over-parameterized models, the loss landscape typically contains broad sets of near-optimal solutions, where distinct parameter configurations can achieve comparable losses and are connected by low-loss paths~\citep{garipov2018loss,draxler2018essentially}. Consequently, independently optimized LoRA updates \(\{(B_t,A_t)\}_{t=1}^T\) can be arbitrarily scattered, as illustrated in Fig.~\ref{fig:mat}(a), making future-domain extrapolation unstable and underdetermined.

The issue does not undermine LoRA, but shows that low-rank adaptation alone is insufficient for temporal generalization. Given a pretrained parameter \(W_{\mathrm{pre}}\), LoRA represents adaptation by a low-rank update \(\Delta W=BA\), where \(B\in\mathbb{R}^{d\times r}\), \(A\in\mathbb{R}^{r\times k}\), and \(\mathrm{rank}(\Delta W)\le r\). Thus, LoRA restricts adaptation to an at-most-rank-\(r\) update subspace, providing a natural low-dimensional representation required by temporal relation modeling among domain-specific updates. Together with Lemmas~\ref{lem:translation_manifold}--\ref{lem:peft_manifold}, this motivates a temporal manifold constrained extension of LoRA for future-domain generalization.

\subsubsection{Subspace-Shared Time-Varying Parameterization}

Let $\{\Delta W_t = B_t A_t\}_{t=1}^T$ be a sequence of LoRA updates where $B_t \in \mathbb{R}^{d \times r}$ and $A_t \in \mathbb{R}^{r \times k}$. Although each update has rank at most \(r\), the union of its factor subspaces can have dimension as large as \(Tr\). Define the global column and row subspaces as $\mathcal{C} = \mathrm{span}(\bigcup_{t=1}^T \mathrm{col}(B_t))$ and $\mathcal{R} = \mathrm{span}(\bigcup_{t=1}^T \mathrm{row}(A_t))$. Choose fixed matrices $B \in \mathbb{R}^{d \times r'}$ and $A \in \mathbb{R}^{r' \times k}$ whose column and row spaces contain $\mathcal{C}$ and $\mathcal{R}$, respectively, with \(r'\geq\max\{\dim(\mathcal{C}),\dim(\mathcal{R})\}\). Since $\mathrm{col}(B_t)\subseteq\mathrm{col}(B)$, there exists a coefficient matrix $C_t \in \mathbb{R}^{r' \times r}$ such that
\begin{equation}
B_t = B C_t
\end{equation}
Similarly, because $\mathrm{row}(A_t)\subseteq\mathrm{row}(A)$, there exists $D_t \in \mathbb{R}^{r \times r'}$ such that
\begin{equation}A_t = D_t A.\end{equation}
Substituting both factorizations gives
\begin{equation}
\Delta W_t = B_t A_t = (B C_t) (D_t A) = B (C_t D_t) A
\end{equation}
Defining $F_t := C_t D_t \in \mathbb{R}^{r' \times r'}$ yields the manifold-constrained factorization, which separates the time-varying update into fixed ambient bases and a dynamic temporal core:
\begin{equation}\label{eq:factorization}
\Delta W(t) = B \cdot F_t \cdot A
\end{equation} 

As illustrated in Fig.~\ref{fig:mat}(b), the shared bases \(B\) and \(A\) define a time-invariant update coordinate system, while \(F_t\) captures domain-specific adaptation. Future-domain generalization then requires extrapolating only this compact core sequence, as detailed in Sec.~\ref{sec:core}. The induced updates \(\Delta W_t = B F_t A\) vary coherently within a shared subspace, avoiding the linear parameter growth of independent domain-wise adapters analyzed in Appendix~\ref{app:complexity}.
The construction is exact when \(r'\) spans the union subspaces. In practice, LoTiC sets \(r'=r\) and learns a compact shared support; Sec.~\ref{sec:justification} gives sufficient conditions under which the domain-wise factor subspaces remain close to such a support.


\subsection{Stability and Justification of the Shared-Basis Form}
\label{sec:justification}
To justify the shared-basis parameterization \(\Delta W_t = B F_t A\), we analyze a sequence of domain-wise LoRA factors connected by gradient updates. Under explicit stability conditions, their factor subspaces do not drift far over time. Concretely, starting from \((B_1,A_1)\), we prove that the components of \(B_t\) orthogonal to \(\mathrm{col}(B_1)\) and of \(A_t\) orthogonal to \(\mathrm{row}(A_1)\) remain small
throughout the updates, up to a controlled leakage term. This establishes that using fixed shared subspaces for \(B\) and \(A\) is meaningful, and that temporal variation can be concentrated into a
small core \(F_t\).

\begin{definition}[Reference subspaces and leakage]
\label{def:reference_subspaces}
Let \(B:=B_1\) and \(A:=A_1\). Define the reference subspaces
\begin{equation}
\mathcal{S}_B := \mathrm{col}(B_1)\subseteq \mathbb{R}^d,\qquad \mathcal{S}_A := \mathrm{row}(A_1)\subseteq \mathbb{R}^k.
\end{equation}
Let \(P_B\) and \(P_A\) be the orthogonal projectors onto \(\mathcal{S}_B\) and \(\mathcal{S}_A\), respectively. For LoRA factors \((B_t,A_t)\), define
\begin{equation}
e_t^B := \|(I-P_B)B_t\|_F,\qquad e_t^A := \|A_t(I-P_A)\|_F.
\end{equation}
\end{definition}

\begin{lemma}[Subspace stability under gradient descent (Preliminary Lemma of Theorem~\ref{thm:BA_subspace_stability_multiass})]
\label{thm:subspace_stability_short_inf}
Suppose the LoRA factors \((B_t,A_t)\) are updated by gradient descent:
\begin{equation}
B_{t+1}=B_t-\eta \nabla_B \mathcal{J}_t(B_t,A_t),\qquad A_{t+1}=A_t-\eta \nabla_A \mathcal{J}_t(B_t,A_t).
\end{equation}
Under the stability assumptions~\ref{ass:gd_dynamics},~\ref{ass:dissipativity} and~\ref{ass:proj_grad_growth} stated in Appendix~\ref{sec:proof1}, if the initial leakage \(e_1^B,e_1^A\) and the gradient leakage terms are small, then \(e_t^B\) and \(e_t^A\) remain uniformly small throughout training. Consequently, for \(\Delta W_t:=B_tA_t\),
\begin{equation}\label{eq:DeltaW_leak_multiass_inf}
\|(I-P_B)\Delta W_t\|_F \le e_t^B\|A_t\|_2,\qquad \|\Delta W_t(I-P_A)\|_F \le \|B_t\|_2 e_t^A.
\end{equation}
Thus, \(\Delta W_t\) has small leakage outside the shared subspace pair \((\mathcal{S}_B,\mathcal{S}_A)\), which justifies using fixed shared subspaces while allowing the coordinates inside them to vary over time.
\end{lemma}

The formal theorem and proof are provided in Appendix~\ref{sec:full_justification}. Lemma~\ref{thm:subspace_stability_short_inf} shows that, under the dissipativity conditions in Assumption~\ref{ass:dissipativity}, the out-of-subspace energies \((e_t^B)^2\) and \((e_t^A)^2\) follow a contractive recursion with bounded additive terms. Hence, if the initialization
is aligned (\(e_1^B=e_1^A=0\)) and leakage is small, then \(\mathrm{col}(B_t)\) stays close to
\(\mathrm{col}(B_1)\) and \(\mathrm{row}(A_t)\) stays close to \(\mathrm{row}(A_1)\) for all \(t\), which
in turn implies that the induced updates \(\Delta W_t=B_tA_t\) largely remain supported on the shared
subspaces, as quantified by \eqref{eq:DeltaW_leak_multiass_inf}. Therefore, a shared \((B,A)\) with a
time-varying core is a consistent structural restriction rather than an ad hoc constraint.

\subsection{Instantiations of the Temporal Core}
\label{sec:core}
Under the shared-basis factorization \(\Delta W_t = B F_t A\), temporal variation is concentrated in the low-dimensional core \(F_t\). When \(B\) has full column rank and \(A\) has full row rank, the map \(F\mapsto BFA\) is injective, so the core provides valid coordinates for the induced update subspace. As illustrated in Fig.~\ref{fig:mat}(b), we parameterize \(F_t\) according to the inductive bias of the underlying dynamics and consider three regimes:

\begin{itemize}[left = 0em]
\item \textbf{Flow-based Core:}
If the data evolution follows a smooth continuous flow governed by an autonomous linear dynamical system, we use a flow-based core. Specifically, the core evolves as
\begin{equation}
    F_t=\exp(tW), \qquad W\in\mathbb{R}^{r'\times r'}.
\end{equation}
Here, \(W\) is learnable, and \(\exp(\cdot)\) denotes the matrix exponential. This form imposes a continuous-time structure on the core trajectory.

\item \textbf{Recurrent Core:}
When temporal evolution is sequential or path-dependent, we parameterize the core with a recurrent model:
\begin{equation}
    h_t = \mathrm{RNN}\big(h_{t-1}, \mathrm{vec}(F_{t-1})\big),
    \qquad
    \mathrm{vec}(F_t) = W_o h_t .
\end{equation}
This formulation captures autoregressive dependencies from previous time steps and is suited to sequential temporal shifts.

\item \textbf{Time-to-Core:}
When temporal domains are indexed by explicit time or other descriptors, we directly parameterize the core as a function of time:
\begin{equation}
    \mathrm{vec}(F_t) = f_{\psi}(t).
\end{equation}
Here, \(f_{\psi}\) is a lightweight time-conditioned network that maps the temporal index to \(F_t\). 
\end{itemize}

We provide a detailed model complexity analysis in Appendix \ref{app:complexity}.

\raggedbottom
\section{Experiments}
\label{sec:experiments}

We organize our experiments around four questions. \textbf{RQ1.} Can LoTiC recover the correct evolution of the decision function? \textbf{RQ2.} Can LoTiC generalize across benchmarks, tasks, modalities, and backbones? \textbf{RQ3.} Can LoTiC remain parameter efficient as the backbone scales? \textbf{RQ4.} How do different temporal parameterizations affect performance?

\paragraph{Datasets.}
We use six established temporal benchmarks. Rotating 2-Moons and Rotated MNIST are standard TDG controlled testbeds~\citep{bai2023temporal,cai2024continuous}. Yearbook, HuffPost, and arXiv are from Wild-Time~\citep{yao2022wildtime}. We further construct a news-continuation task from the temporally indexed WMT News Crawl corpus~\citep{luu2022time}. They span tabular, visual, and textual modalities, classification and generation, and controlled and natural shifts over decades. Dataset details and temporal splits are provided in Appendix~\ref{app:datasets}.

\paragraph{Baselines.}
We compare LoTiC with eight baselines from three groups. \textbf{Practical baselines.} ERM, LastDomain, and IncFT. \textbf{Parameter-dynamics TDG methods.} DRAIN~\citep{bai2023temporal}, Koodos~\citep{cai2024continuous}, and FreKoo~\citep{yu2025learning} represent state-of-the-art approaches. \textbf{Temporal methods for pretrained models.} PromptTDG~\citep{hosseini2025scalable} and TEA~\citep{liu2025scaling} are the latest temporal methods designed for pretrained backbones. For scalability, we apply TDG methods to LoRA parameters while retaining their original temporal mechanisms. Baseline definitions and adaptation details are provided in Appendix~\ref{app:baselines}.

\paragraph{Protocol and metrics.}
For each dataset, we evaluate backbones at different scales, ranging from 66M to 8B parameters overall. We report error rate (\%, lower is better) for classification and ROUGE-L F1 (higher is better) for text generation. Unless stated otherwise, LoTiC uses the flow-based core. All methods are trained on the observed domains and then deployed to all unseen future domains, with the same data splits and training budgets. We repeat each experiment with three seeds and report the mean and standard deviation. Dataset-specific configurations are provided in Appendix~\ref{app:implementation}.

\subsection{Controlled Temporal Extrapolation}
\label{sec:controlled_experiments}
Controlled shifts provide a known temporal trajectory that real-world benchmarks typically do not. We consider Rotating 2-Moons, a binary classification problem with two interleaving point clouds, and Rotated MNIST, a visual classification problem with digit images. In both testbeds, consecutive domains differ by an $18^\circ$ rotation, producing 20 domains over a complete $360^\circ$ cycle. The first 15 domains are used for training, and the remaining five are held out for future-domain evaluation. Figure~\ref{fig:moons}(a) illustrates the Rotating 2-Moons domain sequence, along which a backbone model must evolve its decision function beyond the observed domains.

\begin{figure}[!t]
    \centering
    \includegraphics[width=\textwidth]{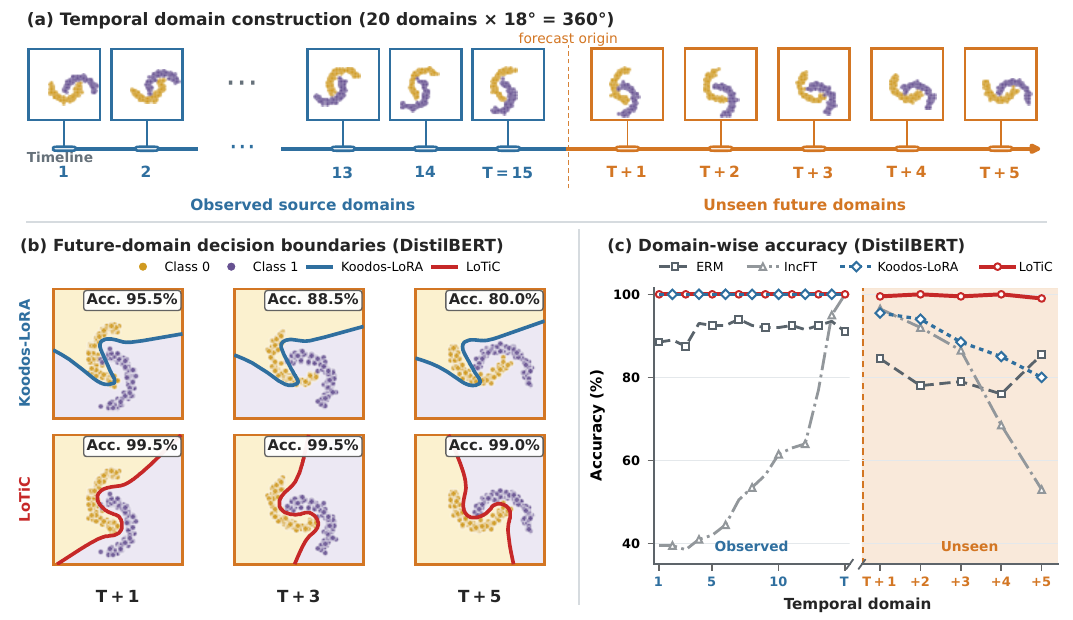}
    \caption{Controlled temporal extrapolation on Rotating 2-Moons. (a) Twenty domains with 15 observed and 5 unseen. (b) Future-domain decision boundaries of Koodos-LoRA and LoTiC. (c) Domain-wise accuracy; the observed and unseen intervals use different x-axis scales.}
    \label{fig:moons}
\end{figure}

Table~\ref{tab:all_benchmarks_preview}(a) summarizes the performance, while Figure~\ref{fig:moons}(b,c) shows how the learned decision functions evolve. \textbf{1) Quantitatively, LoTiC maintains strong performance as the backbone scales to 8B parameters.} LoTiC achieves over 97\% future-domain accuracy on Rotating 2-Moons and reduces the strongest baseline's error by 52.9--54.0\% on Rotated MNIST. Koodos-LoRA and FreKoo-LoRA already incur double-digit error at 66M and become infeasible at larger scales, whereas LoTiC remains near-perfect even on an 8B backbone after observing only 15 domains. \textbf{2) Qualitatively, LoTiC's decision boundary follows the prescribed evolution.} Figure~\ref{fig:moons}(b,c) shows that Koodos-LoRA and LoTiC fit the observed domains equally well, and Koodos-LoRA remains accurate at $T+1$. Yet its boundary then stalls, while LoTiC continues the correct rotation through $T+5$. The contrast lies in where temporal dynamics are modeled. Koodos-LoRA compresses and evolves the complete adapter, whereas LoTiC isolates the evolution within a compact core shared across domains.

\subsection{Generalization Across Benchmarks and Backbones}
\label{sec:benchmark_results}
\begin{table}[!b]
\centering
\caption{Future-domain performance, reported as mean (standard deviation). Classification benchmarks use error rate (\%, $\downarrow$); WMT uses ROUGE-L F1 ($\uparrow$). Bold and underlined values are best and second-best. -L denotes LoRA; OOM denotes out of memory.}
\label{tab:all_benchmarks_preview}
\scriptsize
\captionsetup{skip=3pt}
\newcommand{\mscore}[2]{#1{\tiny\,(#2)}}
\newcommand{\mbest}[2]{\textbf{\mscore{#1}{#2}}}
\newcommand{\msecond}[2]{\underline{\mscore{#1}{#2}}}
\newcommand{\mna}{\textcolor{gray}{\textemdash}}
\newcommand{\moom}{\textcolor{gray}{OOM}}
\newcommand{\mbackbone}[2]{\makebox[8em][l]{#1\hfill\textbf{#2}}}
\newcommand{\mbackbonehead}{\makebox[8em][l]{Backbone\hfill Scale}}
\renewcommand{\arraystretch}{1.02}

\textbf{(a) Controlled benchmarks}\par\vspace{2pt}
\setlength{\tabcolsep}{1.5pt}
\begin{tabular*}{\linewidth}{@{}>{\raggedright\arraybackslash}p{4.7em}!{\vrule width 0.35pt}l!{\vrule width 0.35pt}@{\extracolsep{\fill}}cccccccc@{\hspace{4pt}}c@{}}
\toprule
Dataset & \mbackbonehead & ERM & \makecell{Last\\Domain} & IncFT & DRAIN-L & Koodos-L & FreKoo-L & \makecell{Prompt\\TDG} & TEA & \textbf{LoTiC} \\
\midrule
\multirow{3}{*}{\makecell[l]{\textbf{Rot-2-}\\\textbf{Moons}}}
& \mbackbone{DistilBERT}{66M} & \mscore{18.5}{1.0} & \mscore{12.7}{0.9} & \mscore{28.2}{6.5} & \mscore{21.4}{5.9} & \mscore{11.9}{0.5} & \msecond{11.8}{1.1} & \mscore{25.9}{4.2} & \mscore{21.8}{3.5} & \mbest{0.6}{0.3} \\
& \mbackbone{Qwen3}{0.6B} & \mscore{23.5}{4.4} & \msecond{13.7}{1.3} & \mscore{24.7}{1.0} & \moom & \moom & \moom & \mscore{49.8}{15.9} & \mscore{25.6}{1.3} & \mbest{2.4}{2.5} \\
& \mbackbone{LLaMA-3}{8B} & \mscore{21.8}{6.1} & \msecond{12.7}{0.6} & \mscore{22.1}{7.2} & \moom & \moom & \moom & \mscore{23.4}{5.4} & \mscore{17.3}{1.2} & \mbest{0.9}{0.7} \\
\cmidrule(lr){1-11}
\multirow{2}{*}{\makecell[l]{\textbf{Rot-}\\\textbf{MNIST}}} & \mbackbone{DINOv2-B}{86M} & \mscore{20.0}{1.2} & \mscore{45.9}{2.0} & \mscore{44.7}{2.0} & \mscore{61.8}{2.5} & \mscore{61.3}{5.2} & \mscore{28.7}{1.0} & \mscore{19.1}{1.4} & \msecond{18.7}{0.6} & \mbest{8.6}{2.4} \\
& \mbackbone{DINOv2-L}{0.3B} & \mscore{16.6}{1.9} & \mscore{37.5}{3.3} & \mscore{37.8}{2.7} & \moom & \moom & \moom & \msecond{15.5}{1.1} & \mscore{15.6}{0.4} & \mbest{7.3}{1.0} \\
\bottomrule
\end{tabular*}

\vspace{5pt}
\textbf{(b) Real-world benchmarks}\par\vspace{2pt}
\setlength{\tabcolsep}{1.5pt}
\begin{tabular*}{\linewidth}{@{}>{\raggedright\arraybackslash}p{4.7em}!{\vrule width 0.35pt}l@{\extracolsep{\fill}}cccccccc@{\hspace{5pt}}c@{}}
\toprule
Dataset & Backbone & ERM & LastDomain & IncFT & DRAIN-L & Koodos-L & FreKoo-L & PromptTDG & TEA & \textbf{LoTiC} \\
\midrule
\multirow{2}{*}{\textbf{Yearbook}}
& DINOv2-B & \mscore{13.8}{0.9} & \mscore{11.8}{1.3} & \msecond{11.7}{0.3} & \mscore{17.8}{0.5} & \mscore{13.5}{0.4} & \mscore{16.9}{3.4} & \mscore{13.9}{0.7} & \mscore{14.4}{0.3} & \mbest{9.3}{0.4} \\
& DINOv2-L & \mscore{12.5}{0.9} & \msecond{9.8}{0.5} & \mscore{10.9}{1.0} & \moom & \moom & \moom & \mscore{13.4}{0.3} & \mscore{13.8}{0.6} & \mbest{6.1}{0.6} \\
\cmidrule(lr){1-11}
\multirow{3}{*}{\textbf{HuffPost}}
& DistilBERT & \mscore{29.2}{0.2} & \mscore{29.6}{0.3} & \msecond{29.1}{0.4} & \mscore{29.3}{0.6} & \mscore{30.7}{0.4} & \mscore{30.7}{0.1} & \mscore{29.3}{0.2} & \mscore{30.1}{0.3} & \mbest{28.4}{0.2} \\
& Qwen3 & \mscore{30.7}{0.5} & \mscore{30.6}{0.1} & \mscore{31.8}{0.2} & \moom & \moom & \moom & \mscore{30.0}{0.0} & \msecond{29.7}{0.2} & \mbest{27.9}{0.5} \\
& LLaMA-3 & \mscore{19.3}{0.4} & \msecond{19.1}{0.0} & \mscore{19.5}{0.4} & \moom & \moom & \moom & \mscore{33.7}{21.3} & \mscore{19.3}{0.1} & \mbest{19.0}{0.2} \\
\cmidrule(lr){1-11}
\multirow{3}{*}{\textbf{arXiv}}
& DistilBERT & \mscore{54.1}{0.1} & \mscore{59.9}{0.1} & \mscore{56.9}{0.1} & \mscore{57.6}{0.3} & \mscore{68.8}{8.9} & \mscore{63.1}{0.5} & \msecond{52.6}{0.3} & \mscore{53.4}{0.1} & \mbest{51.8}{0.0} \\
& Qwen3 & \mscore{53.2}{0.4} & \mscore{53.5}{0.1} & \mscore{54.5}{0.1} & \moom & \moom & \moom & \mscore{52.0}{0.1} & \msecond{51.1}{0.3} & \mbest{48.4}{0.3} \\
& LLaMA-3 & \mscore{43.5}{0.3} & \msecond{42.1}{0.2} & \mbest{41.8}{0.2} & \moom & \moom & \moom & \mscore{43.1}{0.2} & \mscore{42.8}{0.2} & \mscore{42.6}{0.2} \\
\cmidrule(lr){1-11}
\multirow{2}{*}{\textbf{WMT}}
& Qwen3 & \mscore{11.4}{0.2} & \mscore{9.5}{0.3} & \mscore{10.7}{0.1} & \moom & \moom & \moom & \mscore{11.2}{0.2} & \msecond{11.4}{0.1} & \mbest{13.0}{0.0} \\
& LLaMA-3 & \mscore{12.2}{0.1} & \mscore{11.8}{0.2} & \mscore{11.5}{0.0} & \moom & \moom & \moom & \msecond{12.3}{0.1} & \mscore{12.2}{0.2} & \mbest{14.3}{0.2} \\
\bottomrule
\end{tabular*}
\end{table}

Table~\ref{tab:all_benchmarks_preview}(b) evaluates LoTiC under natural temporal shifts across four real-world benchmarks. \textbf{1) LoTiC achieves consistent gains across real-world shifts.} It ranks first in nine of ten settings and reduces error by 20.5--37.8\% over the strongest baseline on Yearbook. \textbf{2) The advantage persists as the backbone scales.} LoTiC remains best on two of the three LLaMA-3 settings and is only 0.8 error points behind the best result on arXiv. Larger pretrained models therefore do not remove the value of modeling temporal evolution. \textbf{3) LoTiC transfers across modalities, tasks, and architectures.} The same temporal formulation leads on image classification with DINOv2 and text classification with DistilBERT and Qwen3. On WMT generation, its Qwen3 and LLaMA-3 variants improve ROUGE-L F1 over the strongest baseline by 14.0\% and 16.3\%, respectively.

\subsection{Parameter Efficiency and Scalability}
\label{sec:efficiency_at_scale}
\newcommand{\effhead}[2]{\begin{tabular}[c]{@{}c@{}}#1\\#2\end{tabular}}

We examine two questions about parameter efficiency. \textbf{1)} Can LoTiC improve future-domain performance without the parameter, storage, and runtime costs of existing TDG methods? \textbf{2)} Does its resource use continue to scale like ordinary LoRA as the backbone grows?

\begin{table}[H]
    \centering
    \caption{Efficiency on Rot-2-Moons (DistilBERT, 66M).}
    \label{tab:efficiency_profile}
    \fontsize{8.0}{9.2}\selectfont
    \renewcommand{\arraystretch}{1.10}
    \setlength{\tabcolsep}{3.2pt}
    \begin{tabular*}{0.94\linewidth}{@{\extracolsep{\fill}}l!{\vrule width 0.35pt}rrrr!{\vrule width 0.35pt}rr!{\vrule width 0.35pt}r@{}}
        \toprule
        \multicolumn{1}{c!{\vrule width 0.35pt}}{Method}
        & \multicolumn{1}{c}{\effhead{Total}{params}}
        & \multicolumn{1}{c}{\effhead{Trainable}{params}}
        & \multicolumn{1}{c}{\effhead{Trainable}{vs. LoRA}}
        & \multicolumn{1}{c!{\vrule width 0.35pt}}{\effhead{Ckpt.}{(MB)}}
        & \multicolumn{1}{c}{\effhead{Train}{(s)}}
        & \multicolumn{1}{c!{\vrule width 0.35pt}}{\effhead{Test}{(s)}}
        & \multicolumn{1}{c}{\effhead{Test err.}{(\%)}} \\
        \midrule
        ERM
        & 66.59M & 225,026 & 1.00\(\times\) & 0.900 & 10 & \(<0.1\) & 18.5 \\
        DRAIN--LoRA
        & 182.01M & 115.64M & 513.91\(\times\) & 462.6 & 165.3 & \(<0.1\) & 21.4 \\
        Koodos--LoRA
        & 532.0M & 465.65M & 2069.30\(\times\) & 1862.6 & 580 & \(<0.1\) & 11.9 \\
        FreKoo--LoRA
        & 763.0M & 696.67M & 3095.95\(\times\) & 2786.7 & 150 & \(<0.1\) & 11.8 \\
        \midrule
        \textbf{LoTiC}
        & \textbf{66.59M}
        & \textbf{226,178}
        & \textbf{1.01\(\times\)}
        & \textbf{0.905} & \textbf{82} & \textbf{\(<0.1\)} & \textbf{0.6} \\
        \bottomrule
    \end{tabular*}
    \vspace{-0.6\baselineskip}
\end{table}

\textbf{LoTiC substantially reduces the cost of parameter dynamics.} Table~\ref{tab:efficiency_profile} compares LoTiC with ERM and three parameter-dynamics TDG methods on Rot-2-Moons with DistilBERT. The dominant cost of existing parameter-dynamics methods comes from the machinery used to evolve the adapter. Starting from the same 225K-parameter LoRA adapter, DRAIN--LoRA, Koodos--LoRA, and FreKoo--LoRA expand the trainable state to 514--3,096\(\times\) that of LoRA. By confining the evolution to a compact core, LoTiC limits the trainable-parameter overhead to only 0.51\% and trains 1.8--7.1\(\times\) faster. Parameter dynamics therefore no longer dominate the resource budget.

\begin{wraptable}{r}{0.52\linewidth}
    \vspace{-0.7\baselineskip}
    \centering
    \caption{Parameter overhead by backbone.}
    \label{tab:cross_backbone_overhead}
    \scriptsize
    \renewcommand{\arraystretch}{1.10}
    \setlength{\tabcolsep}{1.15pt}
    \begin{tabular*}{\linewidth}{@{\extracolsep{\fill}}l!{\vrule width 0.35pt}rr!{\vrule width 0.35pt}r@{}}
        \toprule
        Backbone & LoRA & \textbf{LoTiC} & Additional \\
        \midrule
        DistilBERT (66M) & 225,026 & \textbf{226,178} & \textcolor{red!75!black}{+1,152\;(0.51\%)} \\
        DINOv2-B (86M) & 450,058 & \textbf{452,362} & \textcolor{red!75!black}{+2,304\;(0.51\%)} \\
        DINOv2-L (0.3B) & 1,189,898 & \textbf{1,194,506} & \textcolor{red!75!black}{+4,608\;(0.39\%)} \\
        Qwen3 (0.6B) & 1,610,754 & \textbf{1,616,130} & \textcolor{red!75!black}{+5,376\;(0.33\%)} \\
        LLaMA-3 (8B) & 4,739,074 & \textbf{4,745,218} & \textcolor{red!75!black}{+6,144\;(0.13\%)} \\
        \bottomrule
    \end{tabular*}
    \vspace{-0.6\baselineskip}
\end{wraptable}

\textbf{LoTiC preserves LoRA-like parameter efficiency across architectures, modalities, and scales.} Table~\ref{tab:cross_backbone_overhead} shows negligible LoTiC overhead across text encoders, vision transformers, and decoder-only LMs from 66M to 8B. Across this 120\(\times\) range, the temporal component grows only 5.3\(\times\), and its relative overhead falls from 0.51\% to 0.13\%. We next scale a controlled LLaMA-3 family by depth. Figure~\ref{fig:efficiency_scaling}(a) measures peak GPU memory, while Figure~\ref{fig:efficiency_scaling}(b) measures training time relative to the smallest backbone. For Koodos--LoRA, operating on LoRA rather than full-model parameters already reduces the modeled parameter space substantially. Nevertheless, it still learns dynamics over the entire adapter. Its auxiliary dynamics model therefore grows with the adapter dimension and quickly reaches the VRAM limit. LoTiC, however, closely tracks ERM in both memory and training time. Its scaling limit is therefore set by the underlying LoRA training rather than temporal modeling. Appendix~\ref{app:complexity} provides the corresponding complexity analysis.

\begin{figure}[!t]
    \centering
    \includegraphics[width=0.96\linewidth]{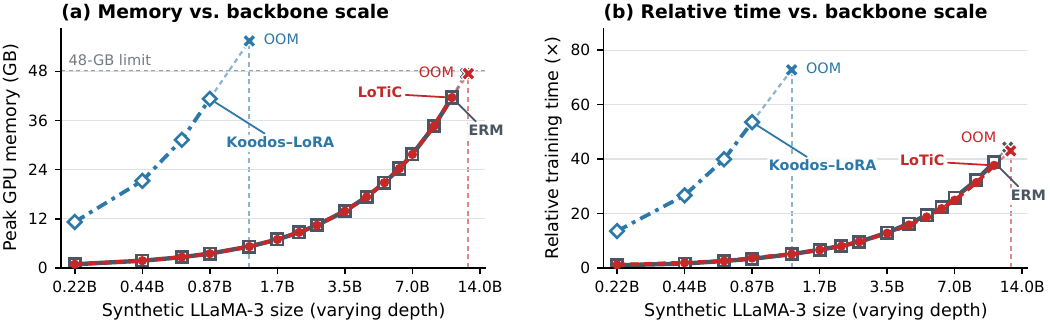}
    \caption{Efficiency at scale on a controlled LLaMA-3 family. Model scale is varied only by increasing the number of identical Transformer blocks. (a) Peak GPU memory. (b) Training time relative to the smallest model. Crosses indicate OOM on a 48-GB GPU.}
    \label{fig:efficiency_scaling}
\end{figure}

\let\effhead\relax

\subsection{Temporal Structure and Parameterization}
\label{sec:temporal_parameterization}

\begin{wraptable}{r}{0.31\linewidth}
    \vspace{-2.5\baselineskip}
    \centering
    \captionsetup{skip=2pt}
    \caption{Yearbook ablation.}
    \label{tab:temporal_parameterization}
    \scriptsize
    \renewcommand{\arraystretch}{0.88}
    \setlength{\tabcolsep}{5.0pt}
    \makebox[\linewidth][c]{%
    \begin{tabular}{@{}l!{\vrule width 0.35pt}r@{}}
        \toprule
        Form & Error (\%, $\downarrow$) \\
        \midrule
        \multicolumn{2}{l}{\textbf{(a) Factorization}} \\
        Static $BA$ (LoRA) & 13.8 (0.9) \\
        $B_tA$      & 33.3 \\
        $BA_t$      & 51.5 \\
        $B_tA_t$    & 51.5 \\
        $\mathbf{BF_tA}$ & \textbf{9.3 (0.4)} \\
        \midrule
        \multicolumn{2}{l}{\textbf{(b) Core model}} \\
        \textbf{Flow} & \textbf{9.3 (0.4)} \\
        Recurrent     & 12.9 \\
        Time-to-Core  & 11.5 (0.3) \\
        \bottomrule
    \end{tabular}}
    \vspace{-2\baselineskip}
\end{wraptable}

Table~\ref{tab:temporal_parameterization} ablates the temporal parameterization on Yearbook with DINOv2-B. \textbf{Isolating temporal variation is critical.} Evolving the high-dimensional factors yields 33.3--51.5\% error, whereas $BF_tA$ improves static LoRA from 13.8\% to 9.3\%. Shared $B$ and $A$ retain the update subspace while only $F_t$ evolves. \textbf{The factorization supports different temporal models.} Recurrent and Time-to-Core achieve 12.9\% and 11.5\%, both outperforming static LoRA. Flow performs best at 9.3\%, indicating that continuous dynamics match the domain evolution.
\par\WFclear
\subsection{Sensitivity and Convergence}
\label{sec:sensitivity_convergence}

\begin{figure}[H]
    \centering
    \includegraphics[width=0.95\textwidth]{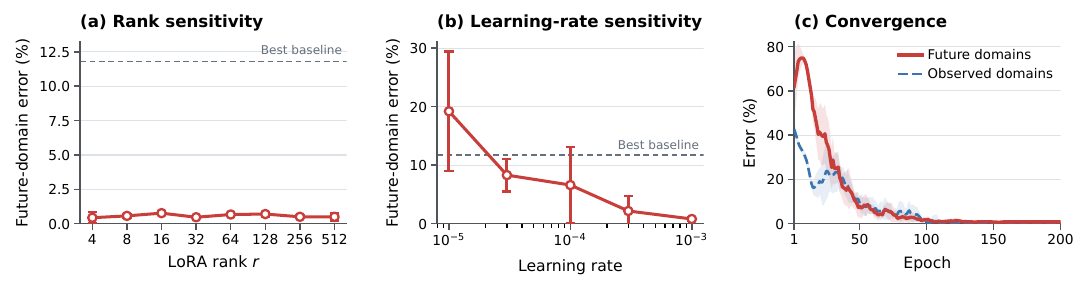}
    \caption{Sensitivity and convergence on Rotating 2-Moons with DistilBERT. Error bars and shading show the standard deviation over three seeds; the dashed line is the strongest baseline.}
    \label{fig:sensitivity_convergence}
\end{figure}

We test whether LoTiC depends on narrow hyperparameter tuning by varying the LoRA rank and learning rate on Rotating 2-Moons with DistilBERT. The sweeps cover a 128-fold range of ranks and a 100-fold range of learning rates. Figure~\ref{fig:sensitivity_convergence}(a,b) reports future-domain error, with the strongest baseline shown as a common reference.

\paragraph{LoRA rank.}
LoTiC remains near 1\% error as the rank varies from 4 to 512. This 128-fold range covers both highly constrained and overparameterized adapters, yet produces little change relative to the gap from the strongest baseline. The gain is therefore not tied to a narrow choice of rank.

\paragraph{Learning rate.}
LoTiC remains below the strongest baseline from $3\!\times\!10^{-5}$ to $10^{-3}$. Smaller learning rates may require longer training horizons, and the degradation at $10^{-5}$ suggests a risk of underfitting when the update is too small.

We record observed- and future-domain errors throughout training under the default configuration. Figure~\ref{fig:sensitivity_convergence}(c) shows consistent convergence across three seeds. Both errors decrease and stabilize near 1\% in later training, while the variation across seeds also narrows. The final performance is therefore sustained over the converged portion of training rather than produced by an isolated checkpoint.

\section{Conclusion}
We introduced LoTiC, a parameter-efficient formulation of temporal domain generalization that confines parameter dynamics to a compact core within shared low-rank bases. This changes the object being evolved from the complete adapter to an $r\times r$ temporal state. Across controlled and real-world shifts, LoTiC generalizes across tasks, modalities, and backbones up to 8B parameters while retaining LoRA-like overhead. Its controlled trajectories and factorization ablations further show that both the shared update subspace and the temporal core are important for learning how a predictor should evolve beyond observed domains. These results show that parameter dynamics need not scale with the backbone to remain effective on future domains.

\clearpage


\bibliographystyle{iclr2027_conference}
\bibliography{reference}

\clearpage

\appendix

\section{Experimental Details and Analyses}
\subsection{\normalfont Datasets}
\label{app:datasets}

Following established TDG evaluation protocols, we combine the widely used Rotating 2-Moons and Rotated MNIST testbeds~\citep{bai2023temporal,cai2024continuous,cai2025continuous,yu2025learning} with three real-world benchmarks from Wild-Time~\citep{yao2022wildtime}: Yearbook, HuffPost, and arXiv. We further construct a temporal news-continuation benchmark from the English WMT News Crawl corpus, which has been used to study temporal adaptation in language models~\citep{luu2022time}. Together, they cover both controlled and real-world temporal shifts, with inputs ranging from tabular data and images to text, and tasks ranging from binary and large-scale multi-class classification to text generation.

\begin{itemize}[leftmargin=1.4em,itemsep=2.5pt,topsep=4pt,parsep=0pt,partopsep=0pt]
    \item \textbf{Rotating 2-Moons.}
    Following the controlled setting used in temporal domain generalization~\citep{bai2023temporal,cai2024continuous}, we construct 20 temporally ordered domains from the two-moons distribution. Each domain contains 200 class-balanced points independently sampled with Gaussian noise of 0.10. Domain $t\in\{0,\ldots,19\}$ is obtained by rotating the distribution counter-clockwise by $18^\circ t$ about the origin. We use the first 15 domains for training and the remaining five as unseen future domains, yielding a complete $360^\circ$ evolution at an $18^\circ$ temporal resolution.

    \item \textbf{Rotated MNIST.}
    We construct a visual counterpart of the controlled benchmark from MNIST~\citep{lecun1998gradient}. A fixed random permutation of the original training split is used to select 4,000 non-overlapping images, with 200 images assigned to each of 20 domains. Images in domain $t$ are rotated clockwise by $18^\circ t$ using bilinear interpolation and zero padding. All ten digit classes are retained. As in Rotating 2-Moons, domains $0$--$14$ are observed during training and domains $15$--$19$ are reserved for future-domain evaluation.

    \item \textbf{Yearbook.}
    Yearbook is a real-world visual benchmark from Wild-Time~\citep{yao2022wildtime} for binary gender classification from $32\times32$ grayscale yearbook portraits. Its temporal shift reflects changes in appearance and fashion across decades. Our processed set contains 37,189 images collected from 1930 to 2013. We merge every two consecutive years into one temporal domain, producing 42 domains in total: 20 observed domains covering 1930--1969 and 22 unseen domains covering 1970--2013. The observed portion contains 14,156 training and 1,594 validation images, while all 21,439 images from the future years are used for evaluation.

    \item \textbf{HuffPost.}
    HuffPost~\citep{yao2022wildtime} studies how news topics and language evolve with current events. The input is a news headline and the task is to predict one of 11 categories that occur throughout the collection. The benchmark contains 63,903 examples grouped into seven annual domains from 2012 to 2018. We use 2012--2015 as the observed period, comprising 35,558 training and 3,948 validation examples, and evaluate on all 24,397 examples from the future domains in 2016--2018.

    \item \textbf{arXiv.}
    The arXiv benchmark~\citep{yao2022wildtime} captures temporal shifts in scientific terminology and topic prevalence. Given a paper title, the task is to predict its primary subject category among 172 classes. We form 16 annual domains spanning 2007--2022 and use a fixed, year-balanced subset of 170,000 examples. The ten domains from 2007--2016 are used for training and validation, and the six domains from 2017--2022 are held out for future-domain evaluation.

    \item \textbf{WMT News Crawl.}
    We construct a temporal news-continuation task from the English WMT News Crawl corpus, following its use in prior temporal adaptation studies~\citep{luu2022time}. Given a 100--200-word news lead, the model predicts the next sentence of 8--26 words. Each year from 2012 to 2021 forms one domain. We use 2012--2017 as observed domains and 2018--2021 as unseen domains.
\end{itemize}

\subsection{\normalfont Baselines and Comparison Protocol}
\label{app:baselines}

We compare LoTiC with eight baselines organized into the same three groups used in Section~\ref{sec:experiments}. Together, they disentangle the effects of time-agnostic adaptation, temporal ordering, explicit parameter dynamics, and mechanisms designed for pretrained backbones.

\paragraph{Practical baselines.}
\textbf{ERM} trains one LoRA adapter and task head on the union of all observed domains, providing a time-agnostic reference with the same backbone adaptation as LoTiC. \textbf{LastDomain} uses the same architecture but trains only on the most recent observed domain, testing whether recency alone explains future performance. \textbf{IncFT} updates one adapter sequentially from the earliest to the latest observed domain without replay. It uses temporal order but does not learn an explicit rule for future model evolution.

\paragraph{Parameter-dynamics TDG methods.}
\textbf{DRAIN~\citep{bai2023temporal}} uses a recurrent generator to evolve model parameters and rolls the generator forward autoregressively for multiple future domains. \textbf{Koodos~\citep{cai2024continuous}} encodes domain-specific parameters into a latent space, learns their continuous evolution through Koopman dynamics, and decodes future parameters. \textbf{FreKoo~\citep{yu2025learning}} separates a parameter trajectory into low- and high-frequency components and extrapolates their dynamics independently.

\begin{itemize}[leftmargin=1.35em,label={},topsep=2pt,itemsep=0pt,parsep=0pt]
\item \textit{Scaling parameter-dynamics baselines.}
DRAIN, Koodos, and FreKoo were originally formulated over full-model parameter trajectories, which makes their temporal states and auxiliary dynamics models prohibitively large for pretrained backbones. We therefore freeze the backbone and use the concatenated LoRA parameters of each observed domain as the dynamic state. This sequence is processed by each method's original temporal mechanism, and the predicted state is reshaped into a future LoRA adapter and attached to the same frozen backbone for evaluation. We denote the resulting implementations as DRAIN-LoRA, Koodos-LoRA, and FreKoo-LoRA. Apart from the necessary input and output dimensions, their temporal mechanisms and rollout procedures are retained.
\end{itemize}

\paragraph{Temporal methods for pretrained models.}
\textbf{PromptTDG~\citep{hosseini2025scalable}} keeps the backbone frozen, learns domain-specific prompts and a temporal prompt generator, and extrapolates a prompt for each future domain. \textbf{TEA~\citep{liu2025scaling}} constructs temporally ordered domain experts, projects their updates into a low-dimensional coordinate space, and obtains the future model by extrapolating and weighting these experts. These baselines represent temporal adaptation mechanisms developed directly for pretrained models rather than full parameter trajectories.

\paragraph{Comparison protocol.}
All methods use identical observed/future domain splits, input preprocessing, pretrained backbones, and task heads. Whenever computationally feasible, LoRA-based methods use the same target modules, rank, and scaling factor as LoTiC; only the method-specific temporal component and objective differ. Future-domain samples and labels are never used for training.

\subsection{\normalfont Implementation and Reproducibility}
\label{app:implementation}

\paragraph{Backbones and prediction heads.}
We use the public DistilBERT, Qwen3--0.6B, LLaMA-3--8B, and DINOv2-B/L checkpoints. Their backbone parameters remain frozen. For Rotating 2-Moons, a trainable linear layer maps each coordinate to an input embedding. DistilBERT uses its first hidden state and decoder-only backbones use the final hidden state. DINOv2 uses the class token for Rotated MNIST and Yearbook. The text classifiers use the corresponding backbone readout, while generation uses the native causal language-modeling head. Prediction heads are identical across methods using the same backbone.

\paragraph{Dataset-specific configurations.}
Hyperparameters are selected using validation data from the observed domains. Dataset- and backbone-specific settings are shared across methods, while method-specific configurations follow the original implementations. Across backbone scales, we retain the same adapter and temporal parameterization, adjusting only numerical settings such as the learning rate, precision, batch size, and LoRA scaling factor when needed for stable training or memory. Table~\ref{tab:dataset_config} gives the shared adapter and base training configurations.

\begin{table}[t]
    \centering
    \caption{Dataset-specific configurations. BF16/FP32 denotes BF16 backbone computation with FP32 trainable modules.}
    \label{tab:dataset_config}
    \scriptsize
    \setlength{\tabcolsep}{3.6pt}
    \begin{tabular*}{\linewidth}{@{}ll@{\extracolsep{\fill}}ccccccc@{}}
        \toprule
        Benchmark & Backbone & Targets & $r$ & $\alpha$ & Precision & LR & Epochs & Batch \\
        \midrule
        \multirow{3}{*}{Rotating 2-Moons}
        & DistilBERT & $q/k/v$ & 8 & 64 & FP32 & $10^{-3}$ & 200 & 200 \\
        & Qwen3--0.6B & $q/k/v$ & 8 & 64 & FP32 & $10^{-3}$ & 500 & 200 \\
        & LLaMA-3--8B & $q/k/v$ & 8 & 128 & BF16/FP32 & $10^{-3}$ & 500 & 200 \\
        \midrule
        Rotated MNIST & DINOv2-B/L & $q/k/v$ & 8 & 64 & FP32 & $10^{-3}$ & 300 & 200 \\
        \midrule
        Yearbook & DINOv2-B/L & $q/k/v/o+\mathrm{FFN}$ & 8 & 32 & FP32 & $3\!\times\!10^{-4}$ & 100 & 128 \\
        \midrule
        \multirow{3}{*}{HuffPost}
        & DistilBERT & $q/k/v/o+\mathrm{FFN}$ & 8 & 64 & FP32 & $5\!\times\!10^{-4}$ & 10 & 128 \\
        & Qwen3--0.6B & $q/k/v/o+\mathrm{FFN}$ & 8 & 64 & FP32 & $5\!\times\!10^{-4}$ & 10 & 128 \\
        & LLaMA-3--8B & $q/k/v/o+\mathrm{FFN}$ & 8 & 64 & BF16/FP32 & $5\!\times\!10^{-5}$ & 10 & 128 \\
        \midrule
        \multirow{3}{*}{arXiv}
        & DistilBERT & $q/k/v/o+\mathrm{FFN}$ & 8 & 64 & FP32 & $5\!\times\!10^{-4}$ & 5 & 128 \\
        & Qwen3--0.6B & $q/k/v/o+\mathrm{FFN}$ & 8 & 64 & FP32 & $5\!\times\!10^{-4}$ & 5 & 128 \\
        & LLaMA-3--8B & $q/k/v/o+\mathrm{FFN}$ & 8 & 64 & BF16/FP32 & $5\!\times\!10^{-5}$ & 5 & 128 \\
        \midrule
        \multirow{2}{*}{WMT News Crawl}
        & Qwen3--0.6B & $q/k/v/o+\mathrm{FFN}$ & 8 & 64 & FP32 & $5\!\times\!10^{-4}$ & 2 & 16 \\
        & LLaMA-3--8B & $q/k/v/o+\mathrm{FFN}$ & 8 & 64 & BF16/FP32 & $5\!\times\!10^{-5}$ & 2 & 16 \\
        \bottomrule
    \end{tabular*}
\end{table}

\paragraph{Method-specific configurations.}
Additional components are configured as follows.
\begin{itemize}[leftmargin=1.2em,itemsep=1pt,topsep=2pt]
    \item \textbf{LoTiC.} Controlled benchmarks use $F_t=\exp(tW)$ with an independent generator for each adapted projection. We initialize $A\sim\mathcal{N}(0,0.01^2)$ and set $B$ and $W$ to zero.
    \item \textbf{DRAIN-LoRA.} We use the original LSTM generator and its official initialization on Rotating 2-Moons. The decoder output initialization is scaled by 0.1 elsewhere.
    \item \textbf{Koodos-LoRA and FreKoo-LoRA.} Both use encoder widths $1024/512/128/32$ followed by linear latent dynamics. Koodos uses loss weights $(1,100,100)$ and segments of length 10. FreKoo uses threshold $0.9$ and weights $(10,1,1)$. Adapter and encoder parameters use learning rate $10^{-3}$.
    \item \textbf{PromptTDG.} Stage 0 trains ERM-LoRA using the corresponding table entry. We then freeze it and learn one prompt token followed by a single-layer, single-head temporal generator.
    \item \textbf{TEA.} We first train ERM-LoRA, and then fine-tune domains in reverse order at one tenth of the base learning rate. We use $c=0.1$, $\epsilon=10^{-3}$, five snapshots, ten Gram-PCA components, ARIMA$(1,1,1)$ extrapolation, and sharpening factor 5.
\end{itemize}

\paragraph{Optimization.}
LoTiC and the LoRA controls use AdamW and a OneCycle schedule. Controlled experiments use weight decay $10^{-4}$, gradient clipping at 1.0, and a warm-up fraction of 0.3. Real-world experiments use zero weight decay, clipping at 0.5, and a warm-up fraction of 0.1. We report the final checkpoint and never select an epoch using future-domain performance.

\paragraph{Evaluation and metrics.}
We report error rate (\%, lower is better) on Rotating 2-Moons, Rotated MNIST, Yearbook, HuffPost, and arXiv, and ROUGE-L F1 (higher is better) on WMT News Crawl. Each experiment uses three random seeds. We first compute the metric for every unseen domain and then average domains with equal weight. The tables report the mean and standard deviation.

\paragraph{Software and hardware.}
We use Python 3.11.16, PyTorch 2.5.1 with CUDA 12.4, Transformers 5.2.0, and NumPy 2.0.1. Experiments are distributed across two servers with seven NVIDIA RTX 6000 Ada GPUs (48\,GB each) in total. The primary server has four such GPUs, two Xeon Gold 6348 processors with 56 physical cores in total, and 1024\,GB of memory.

\subsection{\normalfont Complexity Analysis}
\label{app:complexity}

We analyze the cost introduced by temporal modeling and omit the frozen backbone and task head, which are shared by all methods. Let \(M\) be the number of adapted weight matrices, let the $m$-th matrix have shape $d_m\times k_m$, and let $r$ denote the LoRA rank. We use $T$ for the number of observed domains.

\paragraph{Adapter and core parameters.}
The number of parameters in one shared LoRA adapter is
\begin{equation}
P_{\mathrm{LoRA}}=r\sum_{m=1}^{M}(d_m+k_m).
\end{equation}
For the default flow-based core, LoTiC learns one generator $W_m\in\mathbb{R}^{r\times r}$ for each adapted matrix. Its total parameter count is therefore
\begin{equation}
P_{\mathrm{LoTiC}}=P_{\mathrm{LoRA}}+Mr^2,
\qquad
\frac{P_{\mathrm{LoTiC}}-P_{\mathrm{LoRA}}}{P_{\mathrm{LoRA}}}
=\frac{Mr}{\sum_m(d_m+k_m)}.
\end{equation}
For square projections of width $d$, the relative overhead is $r/(2d)$. Increasing depth makes both terms grow linearly with $M$, while increasing width reduces the relative core overhead. This explains why the measured overhead remains close to LoRA and decreases on wider backbones in Table~\ref{tab:cross_backbone_overhead}.

\paragraph{Cost of parameter dynamics.}
Let $p=P_{\mathrm{LoRA}}$, $h$ denote the hidden width of an auxiliary dynamics model, and $z$ its latent dimension. DRAIN-LoRA, Koodos-LoRA, and FreKoo-LoRA all evolve a state in $\mathbb{R}^{p}$. Their input or output mappings must therefore touch the complete adapter state. LoTiC instead evolves $M$ independent $r\times r$ cores. Table~\ref{tab:complexity_summary} summarizes the resulting leading-order costs.

\begin{table}[H]
    \centering
    \caption{Asymptotic cost of temporal modeling.}
    \label{tab:complexity_summary}
    \small
    \setlength{\tabcolsep}{5pt}
    \begin{tabular}{@{}lccc@{}}
        \toprule
        Method & Evolved state & Auxiliary parameters & Training trajectory \\
        \midrule
        DRAIN--LoRA & $p$ & $\mathcal{O}(ph+h^2)$ & $\mathcal{O}(Tp)$ \\
        Koodos--LoRA & $p$ & $\mathcal{O}(ph+z^2)$ & $\mathcal{O}(Tp)$ \\
        FreKoo--LoRA & $p$ & $\mathcal{O}(ph+z^2)$ & $\mathcal{O}(Tp)$ \\
        LoTiC (Flow) & $r^2$ per matrix & $Mr^2$ & $\mathcal{O}(p+Mr^2)$ \\
        \bottomrule
    \end{tabular}
\end{table}

Thus, applying prior TDG methods to LoRA greatly reduces their state dimension relative to full-model dynamics, but their auxiliary models still scale with the complete adapter. LoTiC changes the modeled object itself. Its temporal component depends only on the low-rank core and is independent of the ambient dimensions $d_m$ and $k_m$. Recurrent and Time-to-Core variants replace $Mr^2$ by the parameter count of their lightweight core generator, without returning to the $p$-dimensional adapter trajectory.

\clearpage
\section{Theoretical Analysis and Proofs}
\label{app:theory}\label{sec:theory}

\subsection{Stability and Justification}
\label{sec:full_justification}

\begin{assumption}[Differentiability and GD dynamics]
\label{ass:gd_dynamics}
For each \(t\in[T]\), the objective \(\mathcal{J}_t:\mathbb{R}^{d\times r}\times\mathbb{R}^{r\times k}\to\mathbb{R}\)
is differentiable. The LoRA factors \((B_t,A_t)\) are updated by gradient descent
 \begin{equation}
B_{t+1}=B_t-\eta \nabla_B \mathcal{J}_t(B_t,A_t),
\qquad
A_{t+1}=A_t-\eta \nabla_A \mathcal{J}_t(B_t,A_t),
 \end{equation}
with step size \(\eta>0\).
\end{assumption}

Assumption~\ref{ass:gd_dynamics} specifies the optimization dynamics used to analyze the LoRA factors. At each training step \(t\), the objective \(\mathcal{J}_t\) may change with the data batch, task, or time index, but it is assumed to be differentiable with respect to both factors \(B_t\) and \(A_t\). This allows the updates of \(B_t\) and \(A_t\) to be written as standard gradient descent steps with a shared step size \(\eta\). The assumption does not impose convexity or a specific loss form; it only fixes the first-order training rule needed for the later subspace stability analysis.

\begin{definition}[Reference shared subspaces]
\label{ass:reference_subspaces}
Let \(B:=B_1\) and \(A:=A_1\). Define the orthogonal projectors
 \begin{equation}
P_B := \mathrm{Proj}_{\mathrm{col}(B)} \in \mathbb{R}^{d\times d},
\qquad
P_A := \mathrm{Proj}_{\mathrm{col}(A^\top)} \in \mathbb{R}^{k\times k}.
 \end{equation}
Equivalently, \(\mathrm{Im}(P_A)=\mathrm{row}(A)\subset\mathbb{R}^k\).
\end{definition}

\begin{assumption}[Dissipativity of out-of-subspace components]
\label{ass:dissipativity}
There exist constants \(\alpha_B,\alpha_A\ge 0\) and \(\varepsilon_B,\varepsilon_A\ge 0\) such that
for all \(t\in[T]\),
\begin{align}
\label{eq:dissip_B}
\big\langle (I-P_B)B_t,  (I-P_B)\nabla_B \mathcal{J}_t(B_t,A_t)\big\rangle
&\ge
\alpha_B  \|(I-P_B)B_t\|_F^2  -  \varepsilon_B,\\
\label{eq:dissip_A}
\big\langle A_t(I-P_A),  \nabla_A \mathcal{J}_t(B_t,A_t)(I-P_A)\big\rangle
&\ge
\alpha_A  \|A_t(I-P_A)\|_F^2  -  \varepsilon_A,
\end{align}
where \(\langle X,Y\rangle := \mathrm{tr}(X^\top Y)\).
\end{assumption}
Assumption~\ref{ass:dissipativity} requires that the optimization dynamics do not amplify the components of \(B_t\) and \(A_t\) outside the selected subspaces. Instead, the gradient acts as a restoring force on \((I-P_B)B_t\) and \(A_t(I-P_A)\), shrinking them whenever they become large, up to bounded slack terms \(\varepsilon_B,\varepsilon_A\). Therefore, the evolution of the in-subspace components and the out-of-subspace components remains approximately decoupled, preserving the intended independence of the structured part from its complement throughout training.

\begin{assumption}[Projected gradient growth outside the reference subspaces]
\label{ass:proj_grad_growth}
There exist constants \(\beta_B,\beta_A\ge 0\) and \(\delta_B,\delta_A\ge 0\) such that for all \(t\in[T]\),
\begin{align}
\label{eq:proj_grad_growth_B}
\|(I-P_B)\nabla_B \mathcal{J}_t(B_t,A_t)\|_F^2
&\le
\beta_B \|(I-P_B)B_t\|_F^2 + \delta_B,\\
\label{eq:proj_grad_growth_A}
\|\nabla_A \mathcal{J}_t(B_t,A_t)(I-P_A)\|_F^2
&\le
\beta_A \|A_t(I-P_A)\|_F^2 + \delta_A.
\end{align}
\end{assumption}
Assumption~\ref{ass:proj_grad_growth} states that the gradient energy outside the reference subspaces is controlled by the current out-of-subspace energy, up to a bounded residual term. Together with Assumption~\ref{ass:dissipativity}, this yields a stable discrete-time recursion under gradient descent: the out-of-subspace components are not only discouraged by the gradient direction, but their one-step fluctuation is also quantitatively bounded.


\begin{theorem}[Stability of \((B_t,A_t)\) subspaces under GD (formal)]
\label{thm:BA_subspace_stability_multiass}
Under Assumptions~\ref{ass:gd_dynamics}--\ref{ass:dissipativity} and
Assumption~\ref{ass:proj_grad_growth}, define
\begin{equation}
e_t^B := \|(I-P_B)B_t\|_F,\qquad
e_t^A := \|A_t(I-P_A)\|_F,\qquad
\Delta W_t := B_tA_t.
\end{equation}
Suppose the step size \(\eta>0\) satisfies
\begin{equation}
\label{eq:stepsize_stability}
0\leq\kappa_B := 1 - 2\eta\alpha_B + \eta^2\beta_B < 1,
\qquad
0\leq\kappa_A := 1 - 2\eta\alpha_A + \eta^2\beta_A < 1.
\end{equation}
Then for all \(t\in[T]\),
\begin{align}
\label{eq:eB_rec_correct}
(e_{t+1}^B)^2
&\le
\kappa_B (e_t^B)^2 + 2\eta\varepsilon_B + \eta^2\delta_B,\\
\label{eq:eA_rec_correct}
(e_{t+1}^A)^2
&\le
\kappa_A (e_t^A)^2 + 2\eta\varepsilon_A + \eta^2\delta_A.
\end{align}
Consequently, for all \(t\in[T]\),
\begin{align}
\label{eq:eB_closed_correct}
(e_t^B)^2
&\le
\kappa_B^{ t-1}(e_1^B)^2
+
\frac{2\eta\varepsilon_B+\eta^2\delta_B}{1-\kappa_B}
\bigl(1-\kappa_B^{ t-1}\bigr),\\
\label{eq:eA_closed_correct}
(e_t^A)^2
&\le
\kappa_A^{ t-1}(e_1^A)^2
+
\frac{2\eta\varepsilon_A+\eta^2\delta_A}{1-\kappa_A}
\bigl(1-\kappa_A^{ t-1}\bigr).
\end{align}
Moreover, the induced update leakage satisfies
\begin{equation}
\label{eq:DeltaW_leak_multiass}
\|(I-P_B)\Delta W_t\|_F \le e_t^B \|A_t\|_2,
\qquad
\|\Delta W_t(I-P_A)\|_F \le \|B_t\|_2 e_t^A.
\end{equation}
\end{theorem}

The detailed proof is provided in Section~\ref{sec:proof1}. 


\subsection{Proof of Theorem~\ref{thm:BA_subspace_stability_multiass}}
\label{sec:proof1}

\begin{proof}
We prove the \(B\)-part; the \(A\)-part follows analogously.

Let \(R_B := I-P_B\). By Assumption~\ref{ass:reference_subspaces}, \(P_B\) is an orthogonal projector, hence \(R_B\) is also an orthogonal projector:
\begin{equation}
R_B^\top = R_B,\qquad R_B^2 = R_B.
\end{equation}
By Assumption~\ref{ass:gd_dynamics},
\begin{equation}
B_{t+1} = B_t - \eta \nabla_B \mathcal{J}_t(B_t,A_t).
\end{equation}
Left-multiplying by \(R_B\), we obtain
\begin{equation}
R_B B_{t+1}
=
R_B B_t - \eta R_B \nabla_B \mathcal{J}_t(B_t,A_t).
\end{equation}
Taking squared Frobenius norms and expanding,
\begin{equation}
\begin{aligned}
\|R_B B_{t+1}\|_F^2
&=
\|R_B B_t\|_F^2
-2\eta \big\langle R_B B_t,  R_B \nabla_B \mathcal{J}_t(B_t,A_t)\big\rangle
+\eta^2 \|R_B \nabla_B \mathcal{J}_t(B_t,A_t)\|_F^2.
\end{aligned}\label{eq:RB_expand_correct}
\end{equation}

By Assumption~\ref{ass:dissipativity}, specifically \eqref{eq:dissip_B},
\begin{equation}
\big\langle R_B B_t,  R_B \nabla_B \mathcal{J}_t(B_t,A_t)\big\rangle
\ge
\alpha_B \|R_B B_t\|_F^2 - \varepsilon_B.
\end{equation}
Substituting this into \eqref{eq:RB_expand_correct} yields
\begin{equation}
\begin{aligned}
\|R_B B_{t+1}\|_F^2
&\le
\|R_B B_t\|_F^2
-2\eta\bigl(\alpha_B \|R_B B_t\|_F^2 - \varepsilon_B\bigr)
+\eta^2 \|R_B \nabla_B \mathcal{J}_t(B_t,A_t)\|_F^2
\\
&=
(1-2\eta\alpha_B)\|R_B B_t\|_F^2
+2\eta\varepsilon_B
+\eta^2 \|R_B \nabla_B \mathcal{J}_t(B_t,A_t)\|_F^2.
\end{aligned}\label{eq:RB_after_dissip_correct}
\end{equation}

Next, by Assumption~\ref{ass:proj_grad_growth}, specifically \eqref{eq:proj_grad_growth_B},
\begin{equation}
\|R_B \nabla_B \mathcal{J}_t(B_t,A_t)\|_F^2
\le
\beta_B \|R_B B_t\|_F^2 + \delta_B.
\end{equation}
Substituting this bound into \eqref{eq:RB_after_dissip_correct}, we obtain
\begin{equation}
\begin{aligned}
\|R_B B_{t+1}\|_F^2
&\le
(1-2\eta\alpha_B)\|R_B B_t\|_F^2
+2\eta\varepsilon_B
+\eta^2\bigl(\beta_B \|R_B B_t\|_F^2 + \delta_B\bigr)
\\
&=
\bigl(1-2\eta\alpha_B+\eta^2\beta_B\bigr)\|R_B B_t\|_F^2
+2\eta\varepsilon_B+\eta^2\delta_B.
\end{aligned}
\end{equation}
By the definition of \(\kappa_B\) in \eqref{eq:stepsize_stability},
\begin{equation}
\|R_B B_{t+1}\|_F^2
\le
\kappa_B \|R_B B_t\|_F^2 + 2\eta\varepsilon_B+\eta^2\delta_B.
\end{equation}
Since \(e_t^B = \|R_B B_t\|_F\), this proves
\begin{equation}
(e_{t+1}^B)^2
\le
\kappa_B (e_t^B)^2 + 2\eta\varepsilon_B+\eta^2\delta_B,
\end{equation}
which is \eqref{eq:eB_rec_correct}.

Now define \(x_t := (e_t^B)^2\) and \(c_B := 2\eta\varepsilon_B+\eta^2\delta_B\). Then
\begin{equation}
x_{t+1} \le \kappa_B x_t + c_B.
\end{equation}
Iterating the recursion gives
\begin{equation}
\begin{aligned}
x_t
&\le
\kappa_B^{ t-1} x_1 + c_B \sum_{j=0}^{t-2}\kappa_B^j
\\
&=
\kappa_B^{ t-1} x_1
+
\frac{c_B}{1-\kappa_B}\bigl(1-\kappa_B^{ t-1}\bigr),
\end{aligned}
\end{equation}
because \(\kappa_B<1\). Therefore,
\begin{equation}
(e_t^B)^2
\le
\kappa_B^{ t-1}(e_1^B)^2
+
\frac{2\eta\varepsilon_B+\eta^2\delta_B}{1-\kappa_B}
\bigl(1-\kappa_B^{ t-1}\bigr),
\end{equation}
which proves \eqref{eq:eB_closed_correct}.

The proof of \eqref{eq:eA_rec_correct} and \eqref{eq:eA_closed_correct} is identical with
\begin{equation}
R_A := I-P_A
\end{equation}
acting on the right, using \eqref{eq:dissip_A} and \eqref{eq:proj_grad_growth_A}.

Finally, since
\begin{equation}
(I-P_B)\Delta W_t = (I-P_B)B_tA_t,
\end{equation}
the submultiplicativity of the Frobenius and spectral norms gives
\begin{equation}
\|(I-P_B)\Delta W_t\|_F
\le
\|(I-P_B)B_t\|_F \|A_t\|_2
=
e_t^B \|A_t\|_2.
\end{equation}
Similarly,
\begin{equation}
\Delta W_t(I-P_A)=B_tA_t(I-P_A),
\end{equation}
hence
\begin{equation}
\|\Delta W_t(I-P_A)\|_F
\le
\|B_t\|_2 \|A_t(I-P_A)\|_F
=
\|B_t\|_2 e_t^A.
\end{equation}
This proves \eqref{eq:DeltaW_leak_multiass}.
\end{proof}

\end{document}